%% file: ProbabilisticPovertyOfKMeans.tex
\documentclass{article}

\usepackage{amsmath,amssymb,amsthm}
\usepackage{amsfonts}

\usepackage{titlesec}

\titleclass{\subsubsubsection}{straight}[\subsection]

\newcounter{subsubsubsection}[subsubsection]
\renewcommand\thesubsubsubsection{\thesubsubsection.\arabic{subsubsubsection}}

\titleformat{\subsubsubsection}
  {\normalfont\normalsize\bfseries}{\thesubsubsubsection}{1em}{}
\titlespacing*{\subsubsubsection}
{0pt}{3.25ex plus 1ex minus .2ex}{1.5ex plus .2ex}

\makeatletter
\renewcommand\paragraph{\@startsection{paragraph}{5}{\z@}%
  {3.25ex \@plus1ex \@minus.2ex}%
  {-1em}%
  {\normalfont\normalsize\bfseries}}
\renewcommand\subparagraph{\@startsection{subparagraph}{6}{\parindent}%
  {3.25ex \@plus1ex \@minus .2ex}%
  {-1em}%
  {\normalfont\normalsize\bfseries}}
\def\toclevel@subsubsubsection{4}
\def\toclevel@paragraph{5}
\def\toclevel@paragraph{6}
\def\l@subsubsubsection{\@dottedtocline{4}{7em}{4em}}
\def\l@paragraph{\@dottedtocline{5}{10em}{5em}}
\def\l@subparagraph{\@dottedtocline{6}{14em}{6em}}
\makeatother

\newtheorem{ax}{Property}
\newtheorem{theorem}{Theorem}
\newtheorem{lemma}{Lemma}

\title{A Note On $k$-Means Probabilistic Poverty}
\author{Mieczys{\l}aw A. K{\l}opotek}
\date{}

\begin{document}
\maketitle
\begin{abstract}
It is proven, by example, that the version of $k$-means with random initialization does not have the property
\emph{probabilistic $k$-richness}.
\end{abstract}
\section{Introduction}
Kleinberg \cite{Kleinberg:2002} coined the term of $k$-richness of distance-based clustering algorithms, meaning the possibility to partition a set of objects into any $k$ non-empty (disjoint) subsets via modifying the distances between these objects. 
However, there exist non-deterministic, probabilistic algorithms which do not fit this characterization because of non-deterministic behaviour. 
Therefore Ackerman at el 
\cite[Definition 3
($k$-Richness)]{Ackerman:2010NIPS}
introduce the concept of \emph{probabilistic $k$-richness}. 
This kind of richness they defined as  
\begin{ax}  \label{ax:k-richness-prob}
For any partition $\Gamma$ of the set $\mathbf{X}$  consisting of exactly $k$ clusters and every $\epsilon>0$ there exists such a distance function $d$ that   the   clustering function  
returns this partition $\Gamma$ with probability exceeding $  1-\epsilon$.   
\end{ax} 

They postulate  in their 
  Fig.2 
(omitting the proof) that    \emph{probabilistic $k$-richness  in probabilistic sense  } is possessed by version of the $k$-means\footnote{
Various versions of $k$-means algorithm are described e.g. in \cite{STWMAK:2018:clustering}.
}
  algorithm with random initialization,
which will be called here $k$-means-random.

Based on a one-dimensional example, we demonstrate that 
 this is not true.

\begin{theorem}\label{th:nokrichnesskmeansrandom}
$k$-means-random algorithm  is not probabilistically $k$-rich  for $k\ge 4$.
\end{theorem}
\begin{proof}
The Theorem follows directly from Lemma \ref{le:four} and Lemma  
\ref{le:moreThanFour}.
The logic is as follows:
There is only a limited number of distinct initializations for a given dataset to be clustered. 
They are picked randomly according to some distribution, in case of $k$-means-random independent of the distances between objects. 
If, for each distance function between the objects, there exists 
 at least one  initialization  for which the expected clustering cannot be found,
then the error $\epsilon$ cannot take on any positive value (close to zero).   
\end{proof}
 
\begin{lemma} \label{le:four} 
In a one-dimensional Euclidean space, given 8 points to cluster into 4 clusters 
each of two neighbouring data points, for each set of distances between data points, 
there exists a $k$-means-random initialization such that the desired clustering is not achieved.  
\end{lemma}
\begin{proof}
See Section \ref{sec:A}.
\end{proof}

\begin{lemma} \label{le:moreThanFour}
In a one-dimensional  Euclidean space, given 2$k$ points to cluster into $k$ clusters ($k>4$)
 each of two neighbouring data points, for each set of distances between data points, 
there exists a $k$-means-random initialization such that the desired clustering is not achieved.  
\end{lemma}
\begin{proof}
See Section \ref{sec:B}.
\end{proof}

\input NONPROBABILISTICKRICHNESS_php_outcome

\bibliographystyle{plain}
\bibliography{../richnessfallacyRKMK/V3_centricconsistencyRAKMAK_bib}

\end{document}

%% file: NONPROBABILISTICKRICHNESS_php_outcome.tex
\section{Proof of Lemma \ref{le:four}} \label{sec:A}
Let us investigate the case when $k=4$.

The proof will consist in investigating relations between node distances and showing that under some special initial seeding (step 1 of $k$-means) there is no chance that a clustering of 8 nodes into 4 pairs can occur. We will consider the following mutually excluding cases: So consider a set of $n=8$ nodes $n_1,\dots,n_8$ arranged in this order on a horizontal straight line from left to right with distances between them denoted as Denote the distances between the points as follows:

$d(n_1,n_2)=a_1$, $d(n_2,n_3)=p_{12}$,
$d(n_3,n_4)=a_2$, $d(n_4,n_5)=p_{23}$,
$d(n_5,n_6)=a_3$, $d(n_6,n_7)=p_{34}$,
$d(n_7,n_8)=a_4$.

This is illustrated symbolically in the figure below.

\par\noindent\rule{\textwidth}{0.4pt} \begin{minipage}{\textwidth} {\footnotesize \begin{verbatim}

n1       n2        n3       n4        n5       n6        n7       n8
[0 --a1-- 0 --p12-- 0 --a2-- 0 --p23-- 0 --a3-- 0 --p34-- 0 --a4-- 0]

\end{verbatim} } \end{minipage} \par\noindent\rule{\textwidth}{0.4pt}

The clustering, that we want to show is impossible under the selected seeding, is the following $\Gamma{}_{0}=\{\{n_1,n_2\},\{n_3,n_4\}, \{n_5,n_6\},\{n_7,n_8\}\}$.

\par\noindent\rule{\textwidth}{0.4pt} \begin{minipage}{\textwidth} {\footnotesize \begin{verbatim}

n1       n2        n3       n4        n5       n6        n7       n8
[0 --a1-- 0]--p12--[0 --a2-- 0]--p23--[0 --a3-- 0]--p34--[0 --a4-- 0]

\end{verbatim} } \end{minipage} \par\noindent\rule{\textwidth}{0.4pt}

By convention, the clusters are delimited with square brackets [] in the figures.

It is obvious that splitting a data set into 4 clusters of two elements can be performed only this way.

We will prove that whatever distances we take, there exists always the possibility of an initial seeding such that $k$-means-random will not find the clustering $\Gamma_0$ we want.

Note that if the clustering $\Gamma{}$ should exist at all, the following must hold:

$|a_1-a_2|< 2p_{12}$,
$|a_2-a_3|< 2p_{23}$,
$|a_3-a_4|< 2p_{34}$,

because otherwise the clusters will take over elements of the neighboring ones. We will consider sharp inequalities only because $k$-means makes a random choice of tiers, hence the probability of a failure seeding is only reduced by a fixed factor and cannot got arbitrarily close to 0 in case of tiers.

So let us proceed case by case.
\subsection{Case  $a_1< p_{12}< a_2$} \label{sec:AA}
Let us investigate the case when $k=4$
AND $a_1< p_{12}< a_2$. Let us choose the seeds (step 1 of $k$-means) $s_1=n_2, s_2=n_4, s_3=5,s_7$. After step 2, the clusters will form: either

\par\noindent\rule{\textwidth}{0.4pt} \begin{minipage}{\textwidth} {\footnotesize \begin{verbatim}

n1       n2        n3       n4        n5       n6        n7       n8
[0 --a1-- * --p12-- 0]--a2--[*]--p23--[*]--a3--[0 --p34-- * --a4-- 0]

\end{verbatim} } \end{minipage} \par\noindent\rule{\textwidth}{0.4pt}

or

\par\noindent\rule{\textwidth}{0.4pt} \begin{minipage}{\textwidth} {\footnotesize \begin{verbatim}

n1       n2        n3       n4        n5       n6        n7       n8
[0 --a1-- * --p12-- 0]--a2--[*]--p23--[* --a3-- 0]--p34--[* --a4-- 0]

\end{verbatim} } \end{minipage} \par\noindent\rule{\textwidth}{0.4pt}

(the asterisks illustrate the seeds). A cluster $\{n_1,n_2,n_3\}$ will form around $s_1$ and the center of this cluster will eventually lie to the right of $n_2$. Hence the next cluster to the right of it will have no possibility to gain control over $n_3$ because it is closer to $n_2$ than to $n_4$. Hence the relation $a_1< p_{12}< a_2$ under appropriate seeding prohibits emerging of $\Gamma_0$, the thesis of the Lemma holds in this case.

By symmetry, it holds also for $a_4< p_{34}< a_3$.
\subsection{Case  $a_1>p_{12}>a_2$} \label{sec:AB}
Let us investigate the case when $k=4$
AND $a_1>p_{12}>a_2$. Assume the following seeding: $s_1=n_1, s_2=n_3, s_3=n_5, s_4=n_7$. After step 2, one of the following clusterings will emerge:

\par\noindent\rule{\textwidth}{0.4pt} \begin{minipage}{\textwidth} {\footnotesize \begin{verbatim}

n1       n2        n3       n4        n5       n6        n7       n8
[*]--a1--[0 --p12-- *]--a2--[0 --p23-- * --a3-- 0]--p34--[* --a4-- 0]

\end{verbatim} } \end{minipage} \par\noindent\rule{\textwidth}{0.4pt}

\par\noindent\rule{\textwidth}{0.4pt} \begin{minipage}{\textwidth} {\footnotesize \begin{verbatim}

n1       n2        n3       n4        n5       n6        n7       n8
[*]--a1--[0 --p12-- *]--a2--[0 --p23-- *]--a3--[0 --p34-- * --a4-- 0]

\end{verbatim} } \end{minipage} \par\noindent\rule{\textwidth}{0.4pt}

\par\noindent\rule{\textwidth}{0.4pt} \begin{minipage}{\textwidth} {\footnotesize \begin{verbatim}

n1       n2        n3       n4        n5       n6        n7       n8
[*]--a1--[0 --p12-- * --a2-- 0]--p23--[* --a3-- 0]--p34--[* --a4-- 0]

\end{verbatim} } \end{minipage} \par\noindent\rule{\textwidth}{0.4pt}

\par\noindent\rule{\textwidth}{0.4pt} \begin{minipage}{\textwidth} {\footnotesize \begin{verbatim}

n1       n2        n3       n4        n5       n6        n7       n8
[*]--a1--[0 --p12-- * --a2-- 0]--p23--[*]--a3--[0 --p34-- * --a4-- 0]

\end{verbatim} } \end{minipage} \par\noindent\rule{\textwidth}{0.4pt}

As visible, the following clusters will form: The first cluster $\{n_1\}$, the second containing at least $n_2,n_3$ and at largest extent also $n_4$, the third at least $n_5$ and the forth at least $n_7,n_8$.

During subsequent iteration the following occurs: The forth cluster keeps $n_7,n_8$ forever. Therefore the third cluster center will be either in the middle of $[n_5,n_6]$ or to the left of it and it will be so as long as the second cluster does not get control over $n_5$. The second cluster center lies to the left of $n_3$. Therefore the first cluster does not get control over $n_2$. Note that the distance of the center of the third cluster to $n_5$ is less than $a_3/2$, and that of the second cluster more than $a_2+p_{23}$. Therefore in the next step the second cluster will not get $n_5$ and so its distance will remain above $a_2+p_{23}$ and it will not change as long as it does not get control over $n_5$, but it cannot and so this will stay forever so. Under these circumstances the distance of the second cluster center to $n_2$ will be smaller than that of the first and so it will stay forever.

Therefore a cluster $\{n_1,n_2\}\in \Gamma{}_0$ cannot form. The thesis of the Lemma holds in this case.

By symmetry same applies to $a_4>p_{34}>a_3$.

We need to check $a_1>p_{12}< a_2$ and $a_1< p_{12}>a_2$
\subsection{Case  $a_1>p_{12}<a_2$} \label{sec:AC}
Let us investigate the case when $k=4$
AND $a_1>p_{12}<a_2$.
\subsubsection{Case: $a_1< (2p_{12}+a_2)/3$} \label{sec:ACA}
Let us investigate the case when $k=4$
AND $a_1>p_{12}<a_2$
AND $a_1< (2p_{12}+a_2)/3$. Consider the seeding $s_1=n_2, s_2=n_4, s_3=n_5, s_4=n_7$. In step 2 one of the clusterings will occur.

\par\noindent\rule{\textwidth}{0.4pt} \begin{minipage}{\textwidth} {\footnotesize \begin{verbatim}

n1       n2        n3       n4        n5       n6        n7       n8
[0 --a1-- * --p12-- 0]--a2--[*]--p23--[* --a3-- 0]--p34--[* --a4-- 0]

\end{verbatim} } \end{minipage} \par\noindent\rule{\textwidth}{0.4pt}

\par\noindent\rule{\textwidth}{0.4pt} \begin{minipage}{\textwidth} {\footnotesize \begin{verbatim}

n1       n2        n3       n4        n5       n6        n7       n8
[0 --a1-- * --p12-- 0]--a2--[*]--p23--[*]--a3--[0 --p34-- * --a4-- 0]

\end{verbatim} } \end{minipage} \par\noindent\rule{\textwidth}{0.4pt}

the first cluster will consist of $n_1,n_2,n_3$ and the second only of $n_4$. In order for the second cluster to gain control over $n_3$, the following condition needs to hold $a_2< (2p_{12}+a_1)/3$ because otherwise the second cluster will never get $n_3$ (as its center will be at $n_4$ or to the rigyht of it). But $a_2< (2p_{12}+a_1)/3< 3a_1/3=a_1$ which contradicts our assumption that $a_1< (2p_{12}+a_2)/3< a_2$. The thesis of the Lemma holds in this case.
\subsubsection{Case: $a_1> (2p_{12}+a_2)/3$} \label{sec:ACB}
Let us investigate the case when $k=4$
AND $a_1>p_{12}<a_2$
AND $a_1> (2p_{12}+a_2)/3$. Assume the following seeding: $s_1=n_1, s_2=n_3, s_3=n_5, s_4=n_7$. Then the following clusters may occur in step 2 o $k$-means.:

\par\noindent\rule{\textwidth}{0.4pt} \begin{minipage}{\textwidth} {\footnotesize \begin{verbatim}

n1       n2        n3       n4        n5       n6        n7       n8
[*]--a1--[0 --p12-- *]--a2--[0 --p23-- *]--a3--[0 --p34-- * --a4-- 0]

\end{verbatim} } \end{minipage} \par\noindent\rule{\textwidth}{0.4pt}

\par\noindent\rule{\textwidth}{0.4pt} \begin{minipage}{\textwidth} {\footnotesize \begin{verbatim}

n1       n2        n3       n4        n5       n6        n7       n8
[*]--a1--[0 --p12-- *]--a2--[0 --p23-- * --a3-- 0]--p34--[* --a4-- 0]

\end{verbatim} } \end{minipage} \par\noindent\rule{\textwidth}{0.4pt}

\par\noindent\rule{\textwidth}{0.4pt} \begin{minipage}{\textwidth} {\footnotesize \begin{verbatim}

n1       n2        n3       n4        n5       n6        n7       n8
[*]--a1--[0 --p12-- * --a2-- 0]--p23--[*]--a3--[0 --p34-- * --a4-- 0]

\end{verbatim} } \end{minipage} \par\noindent\rule{\textwidth}{0.4pt}

\par\noindent\rule{\textwidth}{0.4pt} \begin{minipage}{\textwidth} {\footnotesize \begin{verbatim}

n1       n2        n3       n4        n5       n6        n7       n8
[*]--a1--[0 --p12-- * --a2-- 0]--p23--[* --a3-- 0]--p34--[* --a4-- 0]

\end{verbatim} } \end{minipage} \par\noindent\rule{\textwidth}{0.4pt}

Clusters 3 and 4 will form out of at least nodes $n_5,\dots, n_8$ The forth cluster keeps $n_7,n_8$ forever. Therefore the third cluster center will be either in the middle of $[n_5,n_6]$ or to the left of it and it will be so as long as the second cluster does not get control over $n_5$.

So after centroid update, the second cluster center lies to the left of the middle of $[n_3,n_4]$. Note that the distance of the center of the third cluster to $n_5$ is less than $a_3/2$, and that of the second cluster more than $a_2/2+p_{23}$. Therefore in the next step the second cluster will not get $n_5$ and so its distance will remain above $a_2/2+p_{23}$ and it will not change as long as it does not get control over $n_5$, but it cannot and so this will stay forever so.

The first cluster can capture $n_2$ in the first step only if $a_1< (2p_{12}+a_2)/3$. But we assumed the contrary, that is that $a_1> (2p_{12}+a_2)/3$. So it will never capture it.

The thesis of the Lemma holds in this case.

Therefore, combined with the previous case, thesis of the Lemma holds for $a_1>p_{12}< a_2$ altogether. By symmetry, it holds for $a_4>p_{34}< a_3$ too.
\subsection{Case    $a_1< p_{12}>a_2$} \label{sec:AD}
Let us investigate the case when $k=4$
AND $a_1< p_{12}>a_2$. By symmetry, also $a_4< p_{34}>a_3$, because all the other relations of these distances were already discussed and the Lemma held for them.
\subsubsection{Case: $a_2< p_{23}< a_3$} \label{sec:ADA}
Let us investigate the case when $k=4$
AND $a_1< p_{12}>a_2$
AND $a_2< p_{23}< a_3$. Let us look at the seeding $s_1=n_2, s_2=n_4, s_3=n_6, s_4=n_7$. One of the following clusterings will emerge in step 2.

\par\noindent\rule{\textwidth}{0.4pt} \begin{minipage}{\textwidth} {\footnotesize \begin{verbatim}

n1       n2        n3       n4        n5       n6        n7       n8
[0 --a1-- *]--p12--[0 --a2-- * --p23-- 0]--a3--[*]--p34--[* --a4-- 0]

\end{verbatim} } \end{minipage} \par\noindent\rule{\textwidth}{0.4pt}

Cluster 1 gets $\{n_1,n_2\}$, and cluster 2 gets $\{n_3,n_4,n_5\}$ or Cluster 1 gets $\{n_1,n_2,n_3\}$, and cluster 2 gets $\{n_4,n_5\}$ In subsequent steps Cluster 1 keeps $\{n_1,n_2\}$, but cluster 2 may loose $n_3$ to cluster 1. Hence the distance of the second cluster center to $n_5$ will be equal or smaller than $(2p_{23}+a_2)/3< p_{23}$ hence the third cluster will never gain control over $n_5$ as its distance is at least $a_3$. The thesis of the Lemma holds in this case.
\subsubsection{Case: $ a_2>p_{23}>a_3$} \label{sec:ADB}
Let us investigate the case when $k=4$
AND $a_1< p_{12}>a_2$
AND $ a_2>p_{23}>a_3$. By a symmetric argument, The thesis of the Lemma holds in the case $ a_2>p_{23}>a_3$.

\par\noindent\rule{\textwidth}{0.4pt} \begin{minipage}{\textwidth} {\footnotesize \begin{verbatim}

n1       n2        n3       n4        n5       n6        n7       n8
[0 --a1-- *]--p12--[*]--a2--[0 --p23-- * --a3-- 0]--p34--[* --a4-- 0]

\end{verbatim} } \end{minipage} \par\noindent\rule{\textwidth}{0.4pt}

\subsubsection{Case: $ a_2>p_{23}< a_3$} \label{sec:ADC}
Let us investigate the case when $k=4$
AND $a_1< p_{12}>a_2$
AND $ a_2>p_{23}< a_3$.
\subsubsubsection{Case: $a_2> (2p_{23}+a_3)/3$} \label{sec:ADCA}
Let us investigate the case when $k=4$
AND $a_1< p_{12}>a_2$
AND $ a_2>p_{23}< a_3$
AND $a_2> (2p_{23}+a_3)/3$. Under the seeding $s_1=n_2, s_2=n_3, s_3=n_5, s_4=n_7$,

\par\noindent\rule{\textwidth}{0.4pt} \begin{minipage}{\textwidth} {\footnotesize \begin{verbatim}

n1       n2        n3       n4        n5       n6        n7       n8
[0 --a1-- *]--p12--[*]--a2--[0 --p23-- * --a3-- 0]--p34--[* --a4-- 0]

\end{verbatim} } \end{minipage} \par\noindent\rule{\textwidth}{0.4pt}

the first cluster will form of $n_1,n_2$, the forth of $n_7,n_8$. They will never loose control over these nodes. The second cluster will not capture $n_4$, because its distance to it amounts to $a_2$, and the distance of the third cluster center to it amounts to at most $(2p_{23}+a_3)/3$ which is smaller than $a_2$ by the assumption.

The thesis of the Lemma holds.
\subsubsubsection{Case: $a_2< (2p_{23}+a_3)/3$} \label{sec:ADCB}
Let us investigate the case when $k=4$
AND $a_1< p_{12}>a_2$
AND $ a_2>p_{23}< a_3$
AND $a_2< (2p_{23}+a_3)/3$.

\par\noindent\rule{\textwidth}{0.4pt} \begin{minipage}{\textwidth} {\footnotesize \begin{verbatim}

n1       n2        n3       n4        n5       n6        n7       n8
[0 --a1-- *]--p12--[0 --a2-- *]--p23--[0 --a3-- *]--p34--[* --a4-- 0]

\end{verbatim} } \end{minipage} \par\noindent\rule{\textwidth}{0.4pt}

Under the seeding $s_1=n_2, s_2=n_4, s_3=n_6, s_4=n_7$, the first cluster will form of $n_1,n_2$, the forth of $n_7,n_8$. They will never loose control over these nodes. The third cluster will capture $n_5$ only if $a_3< (2p_{23}+a_2)/3< a_2$. But we assumed $a_2< (2p_{23}+a_3)/3$ which implies that $a_2<(2p_{23}+a_3)/3< a_3$. These two requirements are contradictory. The thesis of the Lemma holds. Combinwed with the former case, thesis of the Lemma holds already when $a_2>p_{23}< a_3$.
\subsubsection{Case: $ a_2< p_{23}>a_3$} \label{sec:ADD}
Let us investigate the case when $k=4$
AND $a_1< p_{12}>a_2$
AND $ a_2< p_{23}>a_3$. This case means that, informally speaking, the gap between clusters has to be bigger than each the distance within the cluster neighbouring with the gap.

Consider the seeding S1: $s_1=n_2, s_2=n_5, s_3=n_7, s_4=n_8$. One of the following clustering may emerge in Step 2.

\par\noindent\rule{\textwidth}{0.4pt} \begin{minipage}{\textwidth} {\footnotesize \begin{verbatim}

n1       n2        n3       n4        n5       n6        n7       n8
[0 --a1-- * --p12-- 0 --a2-- 0]--p23--[* --a3-- 0]--p34--[*]--a4--[*]

\end{verbatim} } \end{minipage} \par\noindent\rule{\textwidth}{0.4pt}

\par\noindent\rule{\textwidth}{0.4pt} \begin{minipage}{\textwidth} {\footnotesize \begin{verbatim}

n1       n2        n3       n4        n5       n6        n7       n8
[0 --a1-- * --p12-- 0]--a2--[0 --p23-- * --a3-- 0]--p34--[*]--a4--[*]

\end{verbatim} } \end{minipage} \par\noindent\rule{\textwidth}{0.4pt}

\par\noindent\rule{\textwidth}{0.4pt} \begin{minipage}{\textwidth} {\footnotesize \begin{verbatim}

n1       n2        n3       n4        n5       n6        n7       n8
[0 --a1-- *]--p12--[0 --a2-- 0 --p23-- * --a3-- 0]--p34--[*]--a4--[*]

\end{verbatim} } \end{minipage} \par\noindent\rule{\textwidth}{0.4pt}

The first cluster captures for sure $n_1,n_2$ and never looses them. If the first cluster would capture $n_1,\dots,n_4$, it would not lose it in further iteration. Thesis holds. So we are left with the clusterings $\Gamma{}_1=\{\{n_1,n_2,n_3\}, \{n_4,n_5,n_6\}, \{n_7\}, \{n_8\}\}$. $\Gamma{}_3=\{\{n_1,n_2\}, \{n_3,n_4,n_5,n_6\}, \{n_7\}, \{n_8\}\}$.

By symmetric seeding S2 $s_1=n_1, s_2=n_2, s_3=n_4, s_4=n_7$. we need to consider only $\Gamma{}_2=\{\{n_1\},\{n_2\},\{n_3, n_4,n_5\},\{n_6,n_7,n_8\}\}$. $\Gamma{}_4=\{\{n_1\},\{n_2\},\{n_3, n_4,n_5,\},\{n_7,n_8\}\}$ after the initialization step.

\par\noindent\rule{\textwidth}{0.4pt} \begin{minipage}{\textwidth} {\footnotesize \begin{verbatim}

n1       n2        n3       n4        n5       n6        n7       n8
[*]--a1--[*]--p12--[0 --a2-- * --p23-- 0]--a3--[0 --p34-- * --a4-- 0]

\end{verbatim} } \end{minipage} \par\noindent\rule{\textwidth}{0.4pt}

\par\noindent\rule{\textwidth}{0.4pt} \begin{minipage}{\textwidth} {\footnotesize \begin{verbatim}

n1       n2        n3       n4        n5       n6        n7       n8
[*]--a1--[*]--p12--[0 --a2-- * --p23-- 0 --a3-- 0]--p34--[* --a4-- 0]

\end{verbatim} } \end{minipage} \par\noindent\rule{\textwidth}{0.4pt}

If after initialization with S2 we would obtain clustering $\Gamma_4$, then this means that $p_{34}>p_{23}+a_3$. On the other hand, it is obvious that if under seeding S1 we obtain any of the clusterings $\Gamma_1$ or $\Gamma_3$, the second cluster will never get node $n_2$, therefore its center weill reside to the riogyht of $n_4$, therefore the third cluster would never capture $n_6$. So the thesis of Lemma holds in this case . So we do not need to comnsider $\Gamma_4$ any more.

By symmwetry, also if $\Gamma_3$ occurs under the seeding S1, the Lemma holds.

So we need to consider S1 leading ton $\Gamma_1$ and S2 leading to $\Gamma_2$.

In case of $\Gamma{}_1$ after S1, in the step 3 can relocate in such a way that step 2 in the next iteration cluster 1 can take over $n_4$.

\par\noindent\rule{\textwidth}{0.4pt} \begin{minipage}{\textwidth} {\footnotesize \begin{verbatim}

n1       n2        n3       n4        n5       n6        n7       n8
[0 --a1-- 0 --p12-- 0 --a2-- 0]--p23--[0 --a3-- 0]--p34--[*]--a4--[*]

\end{verbatim} } \end{minipage} \par\noindent\rule{\textwidth}{0.4pt}

In this case cluster 2 will never regain $n_4$. The Thesis holds.

The other possibility is that instead in the step 2 of the nexct iteration either cluster 2 takes over $n_3$ or $n_3$ remains in cluster 1.

\par\noindent\rule{\textwidth}{0.4pt} \begin{minipage}{\textwidth} {\footnotesize \begin{verbatim}

n1       n2        n3       n4        n5       n6        n7       n8
[0 --a1-- 0 --p12-- 0]--a2--[0 --p23-- 0 --a3-- 0]--p34--[*]--a4--[*]

\end{verbatim} } \end{minipage} \par\noindent\rule{\textwidth}{0.4pt}

\par\noindent\rule{\textwidth}{0.4pt} \begin{minipage}{\textwidth} {\footnotesize \begin{verbatim}

n1       n2        n3       n4        n5       n6        n7       n8
[0 --a1-- 0]--p12--[0 --a2-- 0 --p23-- 0 --a3-- 0]--p34--[*]--a4--[*]

\end{verbatim} } \end{minipage} \par\noindent\rule{\textwidth}{0.4pt}

By analogy in case of $\Gamma{}_2$ after S2 in the next iteration in step 2 we need to consider only either cluster 3 takes over $n_6$ or $n_6$ remains in cluister 4.

\par\noindent\rule{\textwidth}{0.4pt} \begin{minipage}{\textwidth} {\footnotesize \begin{verbatim}

n1       n2        n3       n4        n5       n6        n7       n8
[*]--a1--[*]--p12--[0 --a2-- 0 --p23-- 0]--a3--[0 --p34-- 0 --a4-- 0]

\end{verbatim} } \end{minipage} \par\noindent\rule{\textwidth}{0.4pt}

\par\noindent\rule{\textwidth}{0.4pt} \begin{minipage}{\textwidth} {\footnotesize \begin{verbatim}

n1       n2        n3       n4        n5       n6        n7       n8
[*]--a1--[*]--p12--[0 --a2-- 0 --p23-- 0 --a3-- 0]--p34--[0 --a4-- 0]

\end{verbatim} } \end{minipage} \par\noindent\rule{\textwidth}{0.4pt}

\subsubsubsection{Case: under $\Gamma{}_2$ cluster 3 takes over node $n_6$ and at the same time under $\Gamma{}_1$ cluster 2 takes over node $n_3$} \label{sec:ADDA}
Let us investigate the case when $k=4$
AND $a_1< p_{12}>a_2$
AND $ a_2< p_{23}>a_3$
AND under $\Gamma{}_2$ cluster 3 takes over node $n_6$ and at the same time under $\Gamma{}_1$ cluster 2 takes over node $n_3$. So consider the situation that under $\Gamma{}_2$ cluster 3 takes over node $n_6$

\par\noindent\rule{\textwidth}{0.4pt} \begin{minipage}{\textwidth} {\footnotesize \begin{verbatim}

n1       n2        n3       n4        n5       n6        n7       n8
[*]--a1--[*]--p12--[0 --a2-- 0 --p23-- 0]--a3--[0 --p34-- 0 --a4-- 0]

\end{verbatim} } \end{minipage} \par\noindent\rule{\textwidth}{0.4pt}

\par\noindent\rule{\textwidth}{0.4pt} \begin{minipage}{\textwidth} {\footnotesize \begin{verbatim}

n1       n2        n3       n4        n5       n6        n7       n8
[*]--a1--[*]--p12--[0 --a2-- 0 --p23-- 0 --a3-- 0]--p34--[0 --a4-- 0]

\end{verbatim} } \end{minipage} \par\noindent\rule{\textwidth}{0.4pt}

This implies

$$(2p_{23}+a_2)/3+a_3 <  (2p_{34}+a_4)/3$$
$$ 2p_{23}+a_2  +3a_3 <  (2p_{34}+a_4) $$
$$( 2p_{23}+a_2  +3a_3)/4 <  (2p_{34}+a_4) /4 < 3p_{34}/4 < p_{34}$$

At the same time under $\Gamma{}_1$ let cluster 2 take over node $n_3$.

\par\noindent\rule{\textwidth}{0.4pt} \begin{minipage}{\textwidth} {\footnotesize \begin{verbatim}

n1       n2        n3       n4        n5       n6        n7       n8
[0 --a1-- 0 --p12-- 0]--a2--[0 --p23-- 0 --a3-- 0]--p34--[*]--a4--[*]

\end{verbatim} } \end{minipage} \par\noindent\rule{\textwidth}{0.4pt}

\par\noindent\rule{\textwidth}{0.4pt} \begin{minipage}{\textwidth} {\footnotesize \begin{verbatim}

n1       n2        n3       n4        n5       n6        n7       n8
[0 --a1-- 0]--p12--[0 --a2-- 0 --p23-- 0 --a3-- 0]--p34--[*]--a4--[*]

\end{verbatim} } \end{minipage} \par\noindent\rule{\textwidth}{0.4pt}

This implies:

$$(2p_{23}+a_3)/3+a_2 <  (2p_{12}+a_1)/3 $$
$$(2p_{23}+a_3+3a_2)/4<   p_{12}  $$

The conditions $(2p_{23}+a_3+3a_2)/4< p_{12} $ and $( 2p_{23}+a_2 +3a_3)/4 < p_{34}$ mean that the cluster consisting of nodes $n_3,n_4,n_5,n_6$ will never lose any of its members to the clusters on either of its sides. The thesis holds.
\subsubsubsection{Case: that under $\Gamma{}_2$ cluster 3 takes over node $n_6$ and at the same time under $\Gamma{}_1$ clusters 1 and 2 are not taking over mutually their elements} \label{sec:ADDB}
Let us investigate the case when $k=4$
AND $a_1< p_{12}>a_2$
AND $ a_2< p_{23}>a_3$
AND that under $\Gamma{}_2$ cluster 3 takes over node  $n_6$ and at the same time under $\Gamma{}_1$ clusters 1 and 2 are not taking over mutually their elements.

So consider the situation that under $\Gamma{}_2$ cluster 3 takes over node $n_6$

\par\noindent\rule{\textwidth}{0.4pt} \begin{minipage}{\textwidth} {\footnotesize \begin{verbatim}

n1       n2        n3       n4        n5       n6        n7       n8
[*]--a1--[*]--p12--[0 --a2-- 0 --p23-- 0]--a3--[0 --p34-- 0 --a4-- 0]

\end{verbatim} } \end{minipage} \par\noindent\rule{\textwidth}{0.4pt}

\par\noindent\rule{\textwidth}{0.4pt} \begin{minipage}{\textwidth} {\footnotesize \begin{verbatim}

n1       n2        n3       n4        n5       n6        n7       n8
[*]--a1--[*]--p12--[0 --a2-- 0 --p23-- 0 --a3-- 0]--p34--[0 --a4-- 0]

\end{verbatim} } \end{minipage} \par\noindent\rule{\textwidth}{0.4pt}

This implies

$$(2p_{23}+a_2)/3+a_3 <  (2p_{34}+a_4)/3$$
$$( 2p_{23}+a_2  +3a_3)/4 <  (2p_{34}+a_4) /4$$

At the same time assume that under $\Gamma{}_1$ clusters 1 and 2 are not taking over mutually their elements.

\par\noindent\rule{\textwidth}{0.4pt} \begin{minipage}{\textwidth} {\footnotesize \begin{verbatim}

n1       n2        n3       n4        n5       n6        n7       n8
[0 --a1-- 0 --p12-- 0]--a2--[0 --p23-- 0 --a3-- 0]--p34--[*]--a4--[*]

\end{verbatim} } \end{minipage} \par\noindent\rule{\textwidth}{0.4pt}

\par\noindent\rule{\textwidth}{0.4pt} \begin{minipage}{\textwidth} {\footnotesize \begin{verbatim}

n1       n2        n3       n4        n5       n6        n7       n8
[0 --a1-- 0 --p12-- 0]--a2--[0 --p23-- 0 --a3-- 0]--p34--[*]--a4--[*]

\end{verbatim} } \end{minipage} \par\noindent\rule{\textwidth}{0.4pt}

But this means that

$$( 2p_{23}+a_2  +3a_3)/4 >  ( p_{23}+2a_3)/3$$

Taking into account that $(2p_{34}+a_4) /4 < p_{34}$, we obtain from the above equations that therefore $( p_{23}+2a_3)/3< p_{34}$. This means that under $\Gamma{}_1$ the cluster 3 cannot take over $n_6$ so that the clustering $\Gamma{}_1$ remain stable. The thesis holds.

By symmetry the situation that under $\Gamma{}_1$ cluster 2 takes over node $n_3$ and at the same time under $\Gamma{}_2$ clusters 3 and 4 are not taking over mutually their elements supports the thesis also.
\subsubsubsection{Case: under $\Gamma{}_1$ clusters 1 and 2 are not taking over mutually their elements and at the same time under $\Gamma{}_2$ clusters 3 and 4 are not taking over mutually their elements} \label{sec:ADDC}
Let us investigate the case when $k=4$
AND $a_1< p_{12}>a_2$
AND $ a_2< p_{23}>a_3$
AND under $\Gamma{}_1$ clusters 1 and 2 are not taking over mutually their elements and at the same time under $\Gamma{}_2$ clusters 3 and 4 are not taking over mutually their elements. Consider the following seeding S7: $s_1=n_4, s_2=n_6, s_3=n_7, s_4=n_8$.

\par\noindent\rule{\textwidth}{0.4pt} \begin{minipage}{\textwidth} {\footnotesize \begin{verbatim}

n1       n2        n3       n4        n5       n6        n7       n8
[0 --a1-- 0 --p12-- 0 --a2-- *]--p23--[0 --a3-- *]--p34--[*]--a4--[*]

\end{verbatim} } \end{minipage} \par\noindent\rule{\textwidth}{0.4pt}

We get the clustering $\Gamma{}_7=\{\{n_1,n_2,n_3,n_4\},\{n_5,n_6\},\{n_7\},\{n_8\}\}$. We need to prevent the first cluster to keep these initial elements, therefore the following must hold:

$$(a_1+2p_{12}+3a_2)/4>p_{23}+a_3/2$$
hence
$$(3p_{12}+3a_2)/4>p_{23}+a_3/2$$
$$3/4p_{12}+1/4a_2>p_{23}+a_3/2-a_2/2$$

and by analogy under the seeding: $s_1=n_1, s_2=n_2, s_3=n_3, s_4=n_5$

\par\noindent\rule{\textwidth}{0.4pt} \begin{minipage}{\textwidth} {\footnotesize \begin{verbatim}

n1       n2        n3       n4        n5       n6        n7       n8
[*]--a1--[*]--p12--[* --a2-- 0]--p23--[* --a3-- 0 --p34-- 0 --a4-- 0]

\end{verbatim} } \end{minipage} \par\noindent\rule{\textwidth}{0.4pt}

implying the clustering $\Gamma{}_4=\{\{n_1\},\{n_2\},\{n_3,n_4\},\{n_5,n_6,n_7,n_8\}\}$.

$$3/4p_{34}+1/4a_3>p_{23}+a_2/2-a_3/2$$

Either $a_2/2-a_3/2\ge 0$ or $a_3/2-a_2/2\ge 0$. Assume the latter without restraining the generality Hence

$$3/4p_{12}+1/4a_2>p_{23} $$
$$p_{12} >p_{23}$$

Let us consider the clustering $\Gamma{}_2$. In order for the cluster 2 to capture node $n_3$ the following needs to hold:

$$p_{12}< (2a_2+p_{23})/3< 3p_{23}/3=p_{23}$$

which contradicts the previously derived condition $p_{12} >p_{23}$. This case supports the thesis either.

Hence the violation of $k$-richness in the probabilistic sense for $k=4$ is proven.
\section{Proof of Lemma \ref{le:moreThanFour}} \label{sec:B}
Let us investigate the case when $k>4$. For $k$ greater than 4, consider clustering into $k$ two element clusters. The nodes shall be denoted as above $n_1,\dots,n_{2k}$, the distance between elements of cluster $j$ shall be $a_j$ and the distance between cluster $j$ and $j+1$ shall be denoted by $p_{j,j+1}$. If the clustering should exist at all, the following must hold: $|a_j-a_{j+1}|< 2p_{j,j+1}$.

\par\noindent\rule{\textwidth}{0.4pt} \begin{minipage}{\textwidth} {\footnotesize \begin{verbatim}

n1       n2       n3       n4       n5       n6       n7       n8       
[0 --a1-- 0 --p--- 0 --a2-- 0 --p--- 0 --a3-- 0 --p--- 0 --a4-- 0 --p---...

   n        n        n        n        n        n        n        n        
    2i-3     2i-2     2i-1     2i       2i+1     2i+2     2i+3     2i+4    
... 0 --a -- 0 --p--- 0 --a -- 0 --p--- 0 --a -- 0 --p--- 0 --a -- 0 
         i-1               i                 i+1               i+2

   n        n        n        n        n        n        n        
    2k-6     2k-5     2k-4     2k-3     2k-2     2k-1     2k 
...[0 --p--- 0 --a -- 0 --p--- 0 --a -- 0 --p--- 0 --a -- 0]
                  k-2               k-1               k

\end{verbatim} } \end{minipage} \par\noindent\rule{\textwidth}{0.4pt}

we will refrain from showing indexes of $p$ in the figures as they are self-evident.

Let us look at the situation when $p_{i-1,i}>a_i$. Consider a seeding such that for $j>=i$ $s_j=n_{2j-1}$ for some $i$.

\par\noindent\rule{\textwidth}{0.4pt} \begin{minipage}{\textwidth} {\footnotesize \begin{verbatim}

   n        n        n        n        n        n        n        n        
    2i-3     2i-2     2i-1     2i       2i+1     2i+2     2i+3     2i+4    
... * --a -- 0]--p---[* --a -- 0]--p---[*]--a --[0 --p--- *]--a --[0 
         i-1               i                 i+1               i+2

   n        n        n        n        n        n        n        n        
    2k-7     2k-6     2k-5     2k-4     2k-3     2k-2     2k-1     2k 
...[* --a -- 0]--p---[* --a -- 0]--p---[* --a -- 0]--p---[* --a -- 0]
         k-3               k-2               k-1               k

\end{verbatim} } \end{minipage} \par\noindent\rule{\textwidth}{0.4pt}

This ensures that under no step of $k$-means the $j^{th}$ cluster ($j\ge i$ will contain node $n_{2j+1}$. This can be shown by induction. Directly after seeding, in step 2 of $k$-means, the cluster $k$ will contain at least nodes $n_{2k-1}$ and $n_{2k}$ and maybe $n_{2k-2}$, but it cannot contain the node $n_{2k+1}$ as there is no such node. The cluster $j, i<= j< k$ contains at least the node $n_{2j-1}$, and maybe $n_{2j}$ if not contained in the next cluster, maybe $n_{2j-2}$ if not contained in the previous cluster. It does not contain $n_{2j+1}$ because there is the next seed.

Now consider what happens when cluster centres of clusters emerged this way are computed (step 3 of $k$-means). Define the vector $v_{j,1}$ as one from the centre of cluster $j$ to $n_{2j+1}$. It amounts to at least $[a_j/2+p_{j,j+1}]$ at this moment, for $j>4$. Define the vector $v_{j+1,2}$ as one from $n_{2j+1}$ to the centre of cluster $j+1$. It amounts after the initial step to at most $[a_{j+1}/2]$ ($4< j< k$). Therefore cluster $j$ cannot expand in the next step to capture $n_{2j+1}$, because $|a_j-a_{j+1}|< 2p_{j,j+1}$ implies $-a_j+a_{j+1}< 2p_{j,j+1}$, that is $ a_{j+1}/2 < p_{j,j+1}+a_j/2$. Therefore the vector $v_{j,1}$ will not decrease and $v_{j,2}$ will not increase because $v_{j,1}+v_{j,2}$ is constant for $j$ (distance between $n_{2j+1}$ and $n_{2j-1}$) - this is shown by induction on $j=k-1,k-2,\dots,i$ under the condition that cluster $i-1$ would not capture $n_{2i-1}$. Cluster $i-1$ will not capture $n_{2i-1}$, because the cluster $i$ cannot capture $n_{2i+1}$, therefore its center will lie to the left of $n_{2i}$, therefore its distance to $n_{2i-1}$ will amount to $a_i$ at most, while the distance of cluster $i-1$ center will be at distance of at least $p_{i-1,i}$ from $n_{2i-1}$, and by assumption $p_{i-1,i}>a_i$. We will exploit this partial seeding below.

\subsection{Case    $a_1< p_{12}< a_2$} \label{sec:BA}
Let us investigate the case when $k>4$
AND $a_1< p_{12}< a_2$.

Let us choose the seeds $s_1=n_2, s_2=n_4$ and $s_3,s_4,....,s_k$ at nodes $n_{2j-1}$ (that is to the right of $s_2$. )

\par\noindent\rule{\textwidth}{0.4pt} \begin{minipage}{\textwidth} {\footnotesize \begin{verbatim}

n1       n2       n3       n4       n5       n6       n7       n8       
[0 --a1-- * --p--- 0]--a2--[*]--p---[* --a3-- 0]--p---[* --a4-- 0]--p---...

   n        n        n        n        n        n        n        n        
    2k-7     2k-6     2k-5     2k-4     2k-3     2k-2     2k-1     2k 
...[* --a -- 0]--p---[* --a -- 0]--p---[* --a -- 0]--p---[* --a -- 0]
         k-3               k-2               k-1               k

\end{verbatim} } \end{minipage} \par\noindent\rule{\textwidth}{0.4pt}

A cluster $\{n_1,n_2,n_3\}$ will form around $s_1$ and the center of this cluster will eventually lie to the right of $n_2$. Hence the next cluster to the right of it will have no possibility to gain control over $n_3$ because it is closer to $n_2$ than to $n_4$. Hence the relation $a_1< p_{12}< a_2$ supports the thesis.
\subsection{Case    $a_1>p_{12}>a_2$} \label{sec:BB}
Let us investigate the case when $k>4$
AND $a_1>p_{12}>a_2$. Assume the following seeding: $s_1=n_1,$, $s_j =n_{2j-1}$ for $j=2,\dots,k$.

\par\noindent\rule{\textwidth}{0.4pt} \begin{minipage}{\textwidth} {\footnotesize \begin{verbatim}

n1       n2       n3       n4       n5       n6       n7       n8       
[*]--a1--[0 --p--- *]--a2--[0 --p--- * --a3-- 0]--p---[* --a4-- 0]--p---...

   n        n        n        n        n        n        n        n        
    2k-7     2k-6     2k-5     2k-4     2k-3     2k-2     2k-1     2k 
...[* --a -- 0]--p---[* --a -- 0]--p---[* --a -- 0]--p---[* --a -- 0]
         k-3               k-2               k-1               k

\end{verbatim} } \end{minipage} \par\noindent\rule{\textwidth}{0.4pt}

This is the case discussed above in the introduction where $i=2$. Obviously, cluster 1 cannot take over $n_2$, even if second cluster gets $n_4$ because the second cluster would not get $n_5$. Therefore a cluster $\{n_1,n_2\}\in \Gamma{}$ cannot form.

Therefore the relation $a_1>p_{12}>a_2$ also supports the thesis.
\subsection{Case    $ a_m< p_{m,m+1 }< a_{m+1}$ for some $m$} \label{sec:BC}
Let us investigate the case when $k>4$
AND $ a_m< p_{m,m+1 }< a_{m+1}$ for some $m$. Let us now discuss the case $ a_m< p_{m,m+1 }< a_{m+1}$. Let us look at the seeding $s_j=n_{2j}$ for $j=1,...,m-1$, $s_m=n_{2m}, s_{m+1}=n_{2m+2}$, $s_{j}=n_{2(j)-1}$ for $j=m+2,...,k$.

\par\noindent\rule{\textwidth}{0.4pt} \begin{minipage}{\textwidth} {\footnotesize \begin{verbatim}

n1       n2        n3       n4        n5       n6        n7       n8
[0 --a1-- *]--p12--[0 --a2-- * --p23-- 0]--a3--[*]--p34--[* --a4-- 0]

\end{verbatim} } \end{minipage} \par\noindent\rule{\textwidth}{0.4pt}

\par\noindent\rule{\textwidth}{0.4pt} \begin{minipage}{\textwidth} {\footnotesize \begin{verbatim}

n1       n2        n3       n4        n5       n6        n7       n8
[0 --a1-- * --p12-- 0]--a2--[* --p23-- 0]--a3--[*]--p34--[* --a4-- 0]

\end{verbatim} } \end{minipage} \par\noindent\rule{\textwidth}{0.4pt}

Clusters $1,...,m$ resemble clusters $k,...,i$ from the above case BB (but in reverse order) what ensures that the cluster $m+1$ will never get the node $n_{2m+1}$. Therefore this case has to be rejected. So either $ a_m> p_{m,m+1 }< a_{m+1}$ for each $m$ or $ a_m< p_{m,m+1 }> a_{m+1}$ for each $m$ or
\subsection{Case    $a_m> p_{m,m+1 }< a_{m+1}$ for each $m$} \label{sec:BD}
Let us investigate the case when $k>4$
AND $a_m> p_{m,m+1 }< a_{m+1}$ for each $m$. Let $a_i$ be the longest. Make a seeding $s_j=n_{2j-1}$ for $j=1,...,i$. $s_j=n_{2j-2}$ for $j=i+1,k$. Initially clusters will form: $\{n_1\}$, $\{n_{2j-2},n_{2j-1}\}$ for $j=2,...,k-1$. and $\{n_{2k-2},n_{2k-1},n_{2k}\}$. No cluster $j$ will ever take over node $n_{2j+1}$, as previously stated because $a_1>p_{12}$. The question is if it can take over $n_{2j}$. Considers clusters $i$ and $i+1$. with nodes $\{n_{2i-2},n_{2i-1}\}$ and $\{n_{2i},n_{2i+1}\}$ resp. initially. Clusters $1,..i$ are stable until cluster $i+1$ changes. In the extreme case the cluster $i+1$ can take over $n_{2i+2}$. In this case the distance from the cluster center to $n_{2i}$ amounts to $(2p_{i,i+1}+a_{i+1})/3)a_{i+1}$ which means that cluster $i$ will not get the node $n_{2i}$ so that the required clustering cannot be formed. So this case needs to be rejected.
\subsection{Case    $a_m< p_{m,m+1 }> a_{m+1}$ for each $m$} \label{sec:BE}
Let us investigate the case when $k>4$
AND $a_m< p_{m,m+1 }> a_{m+1}$ for each $m$. Consider an $i=5$. As $p_{45}>a_5>a_5/2$, the 4th cluster will never acquire $n_9$. So it is only possible for cluster 5 to acquire $n_8$ or nodes with lower indexes. If this happens, the probabilistic k-richness definition is violated. If it does not acquire it at any point in time, then k-richness definition is violated dues to conditions described in the case $k=4$ above. If cluster 5 acquires $n_8$, then the argument there can repeated under the condition that $a_4$ is close to zero.

This completes the proof.